\documentclass[review]{elsarticle}

\usepackage{hyperref}

\usepackage{natbib}
\usepackage[T1]{fontenc}
\usepackage{pifont} 
\usepackage[utf8]{inputenc}
\usepackage{lmodern}
\usepackage{graphicx}
\usepackage{caption}
\usepackage{etoolbox}
\usepackage{float}
\usepackage{amsthm}
\usepackage{amssymb} 
\usepackage{mathrsfs}
\usepackage{latexsym}
\usepackage{bbm, dsfont}
\usepackage{listings}
\usepackage{amsmath,amsfonts,amssymb}
\usepackage[table]{xcolor}
\usepackage{colortbl}

\usepackage{geometry}
\usepackage{enumitem}
\usepackage{bbm} 
\usepackage{fancyhdr}
\usepackage{textgreek}
\usepackage{minitoc}
\usepackage{hyperref}
\usepackage{listings}
\usepackage{array,multirow,makecell}
\setcellgapes{1pt}
\makegapedcells
\usepackage{tikz}
\usepackage{hyperref}

\usepackage{tikz,tkz-tab}

\definecolor{mygreen}{rgb}{0,0.6,0}
\definecolor{mygray}{rgb}{0.5,0.5,0.5}
\definecolor{mymauve}{rgb}{0.58,0,0.82}

\usepackage{algorithm}
\usepackage{algpseudocode}

\setcitestyle{authoryear,open={(},close={)}}

\newtheorem{assumption}{Assumption}[section]
\newtheorem{lemma}{Lemma}[section]
\newtheorem{proposition}{Proposition}[section]
\newtheorem{remark}{Remark}[section]

\begin{document}

\begin{frontmatter}

\title{Feature selection with optimal coordinate ascent (OCA)}
\tnotetext[mytitlenote]{Fully documented templates are available in the elsarticle package on \href{http://www.ctan.org/tex-archive/macros/latex/contrib/elsarticle}{CTAN}.}

\author[AISquareConnect,LISIC]{David Saltiel}
\author[AISquareConnect,LAMSADE]{Eric Benhamou}

\address[AISquareConnect]{A.I. SQUARE CONNECT, 35 Boulevard d'Inkermann, 92200 Neuilly sur Seine, France}
\address[LISIC]{LISIC - Universite du Littoral - Cote d’Opale, France}
\address[LAMSADE]{LAMSADE, Université Paris Dauphine, \\ Place du Maréchal de Lattre de Tassigny,75016 Paris, France}

\begin{abstract}
In machine learning, Feature Selection (FS) is a major part of efficient algorithm. It fuels the algorithm and is the starting block for our prediction. In this paper, we present a new method, called Optimal Coordinate Ascent (OCA) that allows us selecting features among block and individual features. OCA relies on coordinate ascent to find an optimal solution for gradient boosting methods score (number of correctly classified samples). OCA takes into account the notion of dependencies between variables forming blocks in our optimization. The coordinate ascent optimization solves the issue of the NP hard original problem where the number of combinations rapidly explode making a grid search unfeasible. It reduces considerably the number of iterations changing this NP hard problem into a polynomial search one. OCA brings substantial differences and improvements compared to previous coordinate ascent feature selection method: we group variables into block and individual variables instead of a binary selection. Our initial guess is based on the k-best group variables making our initial point more robust. We also introduced new stopping criteria making our optimization faster. We compare these two methods on our data set. We found that our method outperforms the initial one. We also compare our method to the Recursive Feature Elimination (RFE) method and find that OCA leads to the minimum feature set with the highest score. This is a nice byproduct of our method as it provides empirically the most compact data set with optimal performance.
\end{abstract}


\end{frontmatter}


\section{Introduction}
Feature selection is also known as variable or attribute selection. It is the selection of a subset of relevant attributes in our data that are most relevant to our predictive modeling problem. 
It has been an active and fruitful field of research and development for decades in statistical learning.
It has proven to be effective and useful in both theory and practice for many reasons: 
enhanced learning efficiency and increasing predictive accuracy (see \cite{Mitra_2002}), model simplification to ease its interpretation and improve performance (see \cite{Almuallim_1994}, \cite{Koller_1996} and \cite{Blum_1997}), shorter training time (see \cite{Mitra_2002}), curse of dimensionality avoidance, enhanced generalization with reduced overfitting, implied variance reduction. Both \cite{HastieEtAl_2009} and \cite{Guyon_2003} are nice references to get an overview of various methods to tackle features selections. 
The approaches followed varies. Briefly speaking, the methods can be sorted into three main categories: Filter method, Wrapper methods and Embedded methods. We developed these three categories in the following section.

%

\subsection{Features selection methods}
\subsubsection{Filter methods}
Filter type methods select variables regardless of the model. These methods suppress the least interesting variables by using ranking techniques as a criteria to select the variables. Once the ranking is done, a threshold is determined in order to select features above it. These methods are very effective in terms of computation time and robust to overfitting. By construction, filter methods may select redundant variables as they do not consider the relationships between variables. To stress this last point, we can present one of the most known criteria, the Pearson correlation coefficient, which is simply the ratio between the covariance and the square root of the two variances: $\operatorname{Cov}(x_i,y) / {\sqrt{\operatorname{Var}(x_i) \operatorname{Var}(y)}}$ with $x_i$ the $ i^{th}$ feature in the model and y the label associated. It is well known that this correlation ranking can only detect linear dependencies between features ant the target label. 

\subsubsection{Wrappers methods}
Wrapper methods evaluate subsets of variables. They thus allow detecting possible interactions between variables. In wrapper methods, a model must be trained to test any subsequent feature subset. Consequently, these methods are iterative and computationally expensive. However, these methods can identify the best performing features set for that specific modeling algorithm. Some known examples of wrapper methods are forward and backward feature selection methods.\\
The backward elimination starts with all features and progressively remove them. At the opposite, the forward selection starts with an empty set and progressively add them. \\
If we have $n$ features, we need to train $n$ classifiers for the first step, then $n-1$ classifiers for the second step and so on. We then have $\frac{n(n+1)}{2}$ training steps for both methods. However, forward selection starts with small features subsets so it can be computationally cheaper if the stopping condition is satisfied early.
One of the State of the art wrappers method is Recursive Feature Elimination (RFE) (see for instance \cite{Mangal_2018} for more details). It first fits a model and removes features  until a pre-determined number of features. Features are ranked through an external model that assigns weights to each features and RFE recursively eliminates features with the least weight at each iteration. One of the main limitation to RFE is that it requires the number of features to keep. This is hard to guess a priori and one may need to iterate much more than the desired number of feature to find an optimal feature set.

\subsubsection{Embedded methods}
Embedded method perform feature selection as a part of the modeling algorithm’s execution. Many hybrid methods are developed to combine the advantages of wrappers and filters methods 

\section{Result of convergence}
In order to motivate our method that relies on coordinate ascent, we recall some theoretical results about the convergence of coordinate ascent optimization. The theory is well understood for the convex case, see \cite{Wright_2015}. The non convex case without gradient which is our example is however much harder as we have local minima issue and mathematical assumptions too weak to be able to prove convergence. However, convergence results under strong convex conditions provide some hint about the efficiency of this method and its convergence rate that is linear. Our proof provided in appendix section is inspired by \cite{Nesterov_2012} with a slight modification as we start by the critical point condition. We also provide the various building block lemma to achieve this proof rapidly. In order to have some meaningful result, we need to make some necessary assumptions for our function $f$ to be minimized. Obviously, even if our final problem is a maximization, it is trivial to turn the minimization program into a maximization one by taking the opposite of the objective function. In this section, we stick to the traditional presentation and examine minimization to make proof reading easier. We examine the following optimization program:
\begin{equation}
\underset{x}{\min} f(x)
\end{equation}

We denote by $e_i$ the traditional vector with $0$ for any coordinate except $1$  for coordinate $i$. It is the vector of the canonical basis.

\begin{assumption} \label{assumption_convergence_prop}
We assume our function $f$ is twice differentiable and strongly convex with respect to the Euclidean norm:
\begin{equation}
\label{strong_convexity}
f(y) \geq f(x) + \nabla f(x)^T (y-x) + \frac{\sigma}{2} \left\| y-x \right\|_2^2 \ \text{ for some $\sigma > 0$ and any } x,y \in \mathbb{R}^n
\end{equation}
We also assume that each gradient's coordinate is uniformly $L_i$ Lipschitz, that is, there exists a constant $L_i$ such that for any $x \in \mathbb{R}^n, t \in \mathbb{R}$
\begin{equation}
\label{lipschitz}
\left| \left[\nabla f(x+t e_i)\right]_i - \left[\nabla f(x) \right]_i \right| \leq L_i \left| t \right|
\end{equation}
We denote by  $L_{\max}$ the maximum of these Lipschitz coefficients :
\begin{equation}
\label{Lmax}
L_{\text{max}} = \underset{i = 1 \ldots n}{max} \ L_i
\end{equation}

We assume that the minimum of $f$ denoted by $f^\star$ is attainable and that the left value of the epigraph with respect to our initial starting point $x_0$ is bounded, that is
\begin{equation}
\label{R0}
\underset{x}{\max} \left\{ \left\|  x - x^\star \right\| : f(x) \leq f(x_0) \right\}  \leq R_0
\end{equation}
\end{assumption}

\begin{remark}
Strong convexity means that the function is between two parabolas. Condition \ref{lipschitz} implies that the Gradient's growth is at most linear. Inequality \ref{R0} States that the function is increasing at infinity.
\end{remark}

\begin{proposition} \label{proposition1}
Under assumption \ref{assumption_convergence_prop}, coordinate ascent optimization (cf. Algorithm \ref{RCA}) converges to the global minimum $f^*$ at a linear rate proportional to $2n L_{\max} R_0^2$, that is
\begin{equation}
\label{theorem1}
\mathbb{E}[ f\left( x_k \right) ] - f^\star \leq \frac{2nL_{\text{max}}R_0^2}{k}
\end{equation}
In addition, for $\sigma > 0$, we have
\begin{equation}
\label{theorem2}
\mathbb{E}[ f\left( x_k \right) ] - f^\star \leq \left( 1 - \frac{ \sigma}{n L_{\max}}\right)^k (f(x_0)-f^\star)
\end{equation}
\end{proposition}

\begin{proof}
The proof is quite simple and given in \ref{proof1}.
\end{proof}

\begin{remark}
Our function to be maximize is obviously not convex. However, a linear rate in the convex case is rather a good performance for the ascent optimization method. Provided the method generalizes which is still under research, this convergence rate is a good hint of the efficiency of this method.
\end{remark}

\section{Method developed}
In many applications, we can regroup features among families. We call these features block variables. Typical example is to regroup variables that are observations of some physical quantity but at a different time (like the speed of the wind measure at different hours for some energy prediction problem, like the price of a stock in an algorithmic trading strategy for financial markets, like the temperature or heart beat of a patient at different time, etc ...). Formally, we can regroup our variables into two sets:
\begin{itemize}[label={\tiny\raisebox{1ex}{\textbullet}}]
    \item the first set encompasses $B_1 \ldots B_n$. These are called block variables of different length $L_i$. Mathematically, the Block variables are denoted by $B_i$ with $B_i$ taking value in $\mathbb{R}^{L_i}\  ,\forall i \in 1 \ldots n$
    \item the second set is denoted $S$ and is a block of $p$ single variables.
\end{itemize}

Graphically, our variables looks like that:
\small
\begin{center}
$
\begin{pmatrix}
	\overbrace{
	\begin{matrix}
	    B_{1,1} & \ldots \ldots & B_{1,n} \\
	    \hline
	    \bullet   & \ldots \ldots & \bullet  \\ 
	     \vdots &  & \vdots  \\ 
	    \bullet  & \ldots \ldots & \bullet 
	\end{matrix}
	}^{B_1}
	&
	\ldots \ldots
	& 
	\overbrace{
	\begin{matrix}
	    B_{n,1} & \ldots \ldots & B_{n,s} \\
	    \hline
	    \bullet  & \ldots \ldots & \bullet  \\
	    \vdots &  & \vdots \\ 
	    \bullet  & \ldots \ldots & \bullet 
	\end{matrix}
	}^{\text{$B_n$}}
	&

	& 
	\overbrace{
	\begin{matrix}
	    S_{1} & \ldots \ldots & S_{p} \\
	    \hline
	    \bullet  & \ldots \ldots & \bullet  \\ 
	    \vdots &  & \vdots \\ 
	    \bullet  & \ldots \ldots & \bullet 
	\end{matrix}
	}^{\text{S}}
\end{pmatrix}
$
\end{center}
\normalsize

In addition, we have $N$ variables split between block variables and single variables, hence $N = N_B +p$ with $N_B = \sum_{i=1}^n L_i$.

Our algorithm works as follows. We first fit our classification model to find a ranking of features importance. The performance is computed with the Gini index for each variable. We then keep the first $k$ best ranked features for each blocks $B_1 \ldots B_n$ in order to find the best initial guess for our coordinate ascent algorithm. Notice that the set of unique variables is not modified during the first step of the procedure. The objective function is the number of  correctly classified samples at each iteration.
We then enter the main loop of the algorithm. Starting with the vector of $\left(k,\ldots ,k, \ \mathbbm{1}_p ^T \right) $ as the initial guess for our algorithm, we perform our coordinate ascent optimization in order to find the set with optimal score and the minimum number of features. The coordinate ascent loop stops whenever we either reach the maximum number of iterations or the current optimal solution has not moved between two steps.

We summarize the algorithm in the pseudo code \ref{OCA}. We denote by $\varepsilon$ the tolerance for the convergence stopping condition. To control early stop, we use a precision variable denoted by $\varepsilon_1, \varepsilon_2$ and two iteration maximum $\text{Iteration max}_1$ and $\text{Iteration max}_2$ that are initialized before starting the algorithm. We also denote $\text{Score}( k_1,\  \ldots ,\  k_n,\ \mathbbm{1}_p)$ to be the accuracy score of our classifier with each $B_i$ block of variables retaining $k_i$ best variables and with single variable all retained.

\begin{algorithm}[H]
\caption{OCA algorithm}\label{OCA}
\begin{algorithmic} 
\State \textbf{J Best optimization}
\State We retrieve features importance from a fitted model
\State We find the index $k^\star$ that gives the best score for variables block of same size $k$: 
\State $k^\star \in \underset{k \in \mathbb{R}^{L_{\text{min}}}}{\text{argmax }} \text{Score}\left( k,\  \ldots ,\  k,\ \mathbbm{1}_p \right)$  \Comment{$L_{\text{min}} =\underset{i \in \mathbb{R}^n}{\text{min }}L_i $}
\State Initial guess : $x^0 = \left(k^\star, \ldots, k^{\star}, \mathbbm{1}_p  \right)$ 
\While{$\text{Score}( x^i )- \text{Score}( x^{i-1} ) | \geq \varepsilon_1  \text{ and  } i  \leq \text{Iteration max}_1 $}
	\State $ x_1^{i} \in \underset{j \in \mathbb{R}^{L_1}}{\text{argmax}} \text{ Score}\left( j,\  x_2^{i-1},\  x_3^{i-1} ,\  \ldots\ ,\  x_n^{i-1}, \mathbbm{1}_p\right)$
	\State ...
	\State $ x_n^{i} \in \underset{j \in \mathbb{R}^{L_n}}{\text{argmax}} \text{ Score}\left( x_1^{i},\  x_2^{i},\  x_3^{i} ,\ \ldots\ , j, \mathbbm{1}_p  \right)$
\vspace{0.1cm}
	\State i += 1
\EndWhile
\State
\State \textbf{Full coordinate ascent optimization}
\State Use previous solutions: $X^* = ( x_1^i, \ldots, x_n^{i},  \mathbbm{1}_p)$ \Comment{i is the last index in previous while loop}
\State $Y^* = \text{Score}\left(X^*\right)$ 

\While{ $|Y-Y^*| \geq \varepsilon_2$ and iteration $\leq \text{Iteration max}_2$ }
	\For{i=1 \ldots N}
		\State $ X = X^* $
		\State $ X_i = $not$\left(X_i^*\right) $			\Comment{not(0) = 1 and not(1) = 0}
		\If{$ \text{Score}\left(X\right) \geq  \text{Score}\left(X^*\right)$} 
			\State $ X^* = X $
		\EndIf
	\EndFor
	\State $ Y = \text{Score}\left(X^*\right) $
	\State iteration += 1
\EndWhile
\State Return $X^*, Y^*$
\end{algorithmic}
\end{algorithm}

\begin{remark}
The originality of this coordinate ascent optimization is to regroup variable by block, hence it reduces the number of iterations compared to Binary Coordinate Ascent (BCA) as presented in \cite{ZARSHENAS_2016}
The stopping condition can be changed to accommodate for other stopping conditions.
\end{remark}

\begin{remark}
There are many variants to this algorithm. It can be modified by using a randomized coordinate ascent. In this case, we choose the index randomly at each step instead of using the provided order. The pseudo code is listed below:
\begin{algorithm}
\caption{Randomized Coordinate ascent :}
\label{RCA}
\begin{algorithmic} 
\State \textbf{Initialization}
\State Start with $x_0 \in \mathbb{R}^n$ 
\State Set $k =  0$
\While{stop criteria not satisfied}
	\State Choose index $i_k$ uniformly distributed in  $\{ 1, \ldots,n \}$ independently from prior iteration
	\State Set $x_{k+1} = x_k - \alpha_k \left[ \nabla f\left(x_k \right) \right]_{i_k} e_{i_k}$ for some $\alpha_k > 0$
	\State Set $k = k+1$
\EndWhile
\end{algorithmic}
\end{algorithm}
\end{remark}

\begin{remark}
The specificity of our method is to keep the $j$ best representative features for each feature class, as opposed to other methods that only select one representative feature from each group, ignoring the strong similarities between each feature of a given variable block. This takes in particular the opposite view of feature Selection with Ensembles, Artificial Variables, and Redundancy Elimination as developed in \cite{Tuv_2009}.
\end{remark}

\section{Numerical Results}
\subsection{Data set}
We use this algorithm to do a supervised classification of a data set obtained from financial markets trades that we want to classify according to some a priori features. We are given 
1500 trades with 135 features that can be classified into 5 block of 20 variables,1 block of 30 variables and 5 single variables. We know for each trade whether it is a good or bad trade. The idea is to use the minimum number of features to classify a priori this data set. We use cross validation with 70\% for the training set and 30\% for the test sets. For full reproducibility, full data set and corresponding python code for this algorithm is available publicly on \href{https://github.com/davidsaltiel/OCA.git}{github}.. The authors may further update the code to reflect improvements or typos if required. This code is provided as it is. The authors do not grant any warranty nor assume any liability for the content thereof. 

\subsection{Comparison}
We compare our method to two other methods that are supposed to be State of the art for feature selections, namely RFE and BCA. Our new method achieves a score of 62.80  \%  with 16\% of features used, to be compared to RFE that achieves 62.80 \%  with 19\% of features used. BCA performs poorly with a highest score given by 62.19 \% with 27\% of features used. If we take in terms of efficiency criterium, the highest score with the less feature, our method is the most efficient among these three methods. In comparison, with the same number of features, namely 16\%, RFE gets a score of 62.40 \%. All these figures are summarized in the table \ref{tab:comparison}.

\begin{table}[H]
	\centering
	\caption{Method Comparison: for each row, we provide in red the best(s) (hotest) method(s) and in blue the worst (coldest) method, while intermediate methods are in orange. We can notice that OCA achieves the higher score with the minimum feature sets. For the same feature set, RFE performs worst or equally, if we want the same performance for RFE, we need to have a larger feature set. BCA is the worst method both in terms of score and minimum feature set.}
	\resizebox{\textwidth}{!}{
	\begin{tabular}{|c|c|c|c|c|}
		\hline
		Method & OCA using 24 features & RFE using 24 features  & BCA using 39 features & RFE using 28 features\\
		\hline
		\% of features &\textcolor[rgb]{1,0, 0}{16.6} & \textcolor[rgb]{ 1,0,0}{16.6} & \textcolor[rgb]{ 0,0,0.8}{27.08}&\textcolor[rgb]{1,0.4, 0}{19.4}\\
		\hline
		Score (in \%) &\textcolor[rgb]{1, 0, 0}{62.8} & \textcolor[rgb]{1,0.4, 0}{62.39} & \textcolor[rgb]{0,0,0.8}{62.19}&\textcolor[rgb]{1,0,0}{62.8}\\
		\hline
	\end{tabular}
	}
	\label{tab:comparison}%
\end{table}%

\section{Discussion}
Compared to BCA our method reduces the number of iterations as it uses the fact that variables can be regrouped into categories or classes. Below is provided the number of iterations for OCA and BCA in figure \ref{convergence}. Our method requires only 350 iterations steps ton converge as opposed to BCA that needs up to 700 iterations steps as it computes blindly variables ignoring similarities between the different variables.

\begin{center}
\begin{figure}[H]
       \includegraphics[scale=0.5]{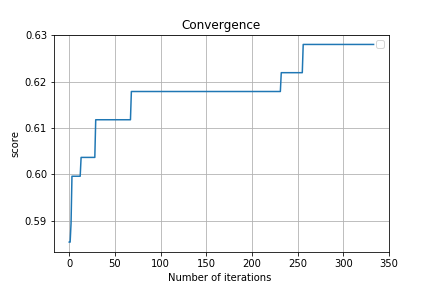}
	 \includegraphics[scale=0.5]{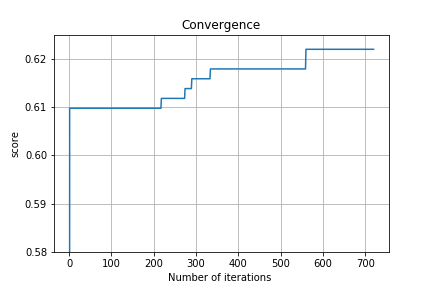}
	\caption{Iterations steps up to convergence for OCA and BCA. OCA method is on the left while BCA is on the right. We see that OCA requires around 350 iteration steps to converge while BCA requires the double around 700 iteration steps to converge}\label{convergence}
\end{figure}
\end{center}

Graphically, we can compute the best candidates for the four methods listed in table \ref{tab:comparison} in figure \ref{fig:comparison} and \ref{fig:comparison2}. We have taken the following color code. The hottest (or best performing) method is plotted in red, while the worst in blue. Average performing methods are plotted in orange. In order to compare finely OCA and RFE, we have plotted in figure \ref{fig:comparison2} the result of RFE for used features set percentage from 10 to 30 percent. We can notice that for the same feature set as OCA, RFE has a lower score and equally that to get the same score as OCA, RFE needs a large features set.

\begin{center}
	\captionof{figure}{Comparison between the 4 methods. To qualify the best method, it should be in the upper left corner. The desirable feature is to have as little features as possible and the highest score. We can see that the red cross that represents OCA is the best. The color code has been designed to ease readability. Red is the best, orange is a slightly lower performance while blue is the worst.}
       \includegraphics[scale=0.5]{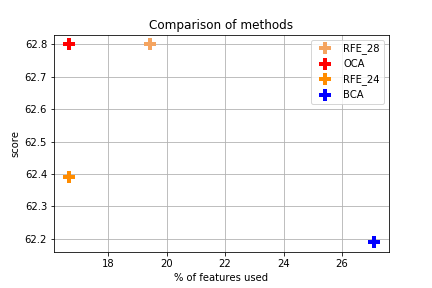}
	\label{fig:comparison}
\end{center}

\begin{center}
	\captionof{figure}{Comparison between OCA and RFE. Zoom on the methods. For RFE, we provide the score for various features set in blue. The two best RFE performers points are the orange cross marker points that are precisely the one listed in table \ref{tab:comparison}. The red cross marker point represents OCA. It achieves the best efficiency as it has the highest score and the smallest feature set for this score.}
       \includegraphics[scale=0.5]{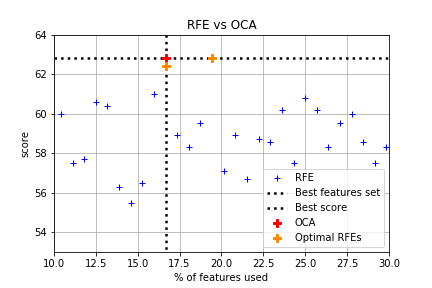}
	\label{fig:comparison2}
\end{center}

\section{Conclusion}
In this paper, we have presented a new method, called Optimal Coordinate Ascent (OCA) that allows us selecting features among block and individual features. OCA relies on coordinate ascent to find an optimal solution for gradient boosting methods score (number of correctly classified samples). OCA takes into account the notion of dependencies between variables forming blocks in our optimization. The coordinate ascent optimization solves the issue of the NP hard original problem where the number of combinations rapidly explode making a grid search unfeasible. It transforms  the NP hard problem of finding the best features into a polynomial search one. Comparing result with two other methods Binary Coordinate Ascent (BCA) and Recursive Feature Elimination (RFE), we find that OCA leads to the minimum feature set with the highest score. OCA provides empirically the most compact data set with optimal performance. Obtaining a reduced features set compared to other method is highly desirable for at least two reasons: First, a lower feature set should have a stronger generalization power as it has less noise created by too many variables (in a similar way in a sense as the Lasso method that eliminates variables in regression). Second, fewer features leads to smaller memory size model and faster computation. Possible extension is to parallelize and potentially use GPU acceleration for this algorithm to leverage its strong decoupling when examining candidate solutions.
\newpage

\appendix

\section{Proofs}

\subsection{Proof of proposition \ref{proposition1} \label{proof1}}
In order to prove result, we first start by a simple lemma.

\begin{lemma}
\label{lemma1}
If $\left(u_n\right)_{n \in \mathbf{N}}$ is non increasing such that $ u_n - u_{n+1} \geq a u_n^2$ with $a >0$\\,
then $u_n \leq \frac{1}{n a}$
\end{lemma}

\begin{proof}
We remark that $ u_{n+1} \leq  u_n - a u_n^2 $ or equivalently $u_{n+1} \leq u_n \left( 1 - a u_n \right)$, which says that $u_{n+1}$ is bounded by a lower parabola. Let $f : x \rightarrow x \ (1-ax)$ be this parabola. $f$'s variation are easy to study and given below (with $f'(x) = 1 - 2 a x$):
\begin{center}
\begin{tikzpicture}
   \tkzTabInit{$x$ / 1 , $f'(x)$ / 1,$f(x)$/1.5}{$-\infty$, $\frac{1}{2a}$, $+\infty$}
   \tkzTabLine{, +, 0, -, }
   \tkzTabVar{-/ , +/$ \frac{1}{4a}$, -/  }
\end{tikzpicture}
\end{center}

We can now trivially prove our result by induction. The initialization step is obvious for $k \leq 4$ as $u_k \leq \frac{1}{ka}$ since the global maximum of our parabola is $1 / 4a$ which is less than $1 / ka$ for $k \leq 4$. If the result holds for $k \geq 2$, we know that the maximum of the parabola ($f(U_k)$) is attained in $\frac{1}{ka}$ since $\frac{1}{ka} \leq \frac{1}{2a}$ . 
This implies that 
\begin{eqnarray*}
 u_{k+1}   &  \leq  & \frac{1}{ka} - a \frac{1}{(ka)^2  }  \qquad \text{or} \qquad u_{k+1}  \leq   \frac{k-1}{k^2a}
\end{eqnarray*}
We can trivially conclude as $\frac{k-1}{k^2} \leq \frac{k-1}{k^2-1} = \frac{1}{k+1}$
\end{proof}

\begin{proof}
\noindent We can now prove our main result. By assumptions, we do a gradient descent according to one coordinate: $x_{k+1} = x_k - \alpha_k \left[ \nabla f \left( x_k \right)  \right]_{i_k} e_{i_k}$ for some $  \alpha_k \geq 0 $. A Taylor-Lagrange expansion for $f(x_{k+1})$'s gradient gives us:
\begin{eqnarray}
  \nabla  f \left( x_{k+1} \right)   & \triangleq  & \nabla  f \left( x_k - \alpha_k \left[ \nabla f \left( x_k \right)  \right]_{i_k} e_{i_k} \right)  \\
   & = &   \nabla  f \left( x_k\right) - \langle \alpha_k \left[ \nabla f \left( x_k \right)  \right]_{i_k} e_{i_k} \  , \  \mathcal{H} f\left( \theta_k x_k  + \left( 1 - \theta_k \right ) x_{k+1}\right) \rangle  \text{ for $ \theta_k \in \left]0,1\right[$ } \label{test}
\end{eqnarray}

\noindent We denote $C_k = \theta_k x_k  + \left( 1 - \theta_k \right ) x_{k+1}$. We want $  \alpha_k$ such that $ \nabla  f \left( x_{k+1} \right) = 0$. Combined with \eqref{test}, we have the following equality: \\
$\alpha_k \langle  \left[ \nabla f \left( x_k \right)  \right]_{i_k} e_{i_k} \  , \  \mathcal{H} f\left( C_k\right) \rangle  = \nabla  f \left( x_k\right)$  \\

\noindent Taking the norm and using the Cauchy Schwarz inequality, we have that : \\
$ \left\| \nabla f \left( x_k \right)  \right\| \leq |\alpha_k| \left\| \left[\nabla f \left( x_k \right) \right]_{i_k} e_{i_k} \right\| \, \left\| \mathcal{H} f \left( C_k \right)  \right\|$ 

\noindent Bounding the Hessian from\eqref{lipschitz} and using equation \eqref{Lmax}, we have:  $ \left\| \nabla f \left( x_k \right)  \right\| \leq |\alpha_k| \left\| \nabla f \left( x_k \right)  \right\|  \, L_{\max}$

\noindent Assuming that the objective function minimum is not attained at step k, $ \left\| \nabla f \left( x_k \right)  \right\| \neq 0 $, we have:
$\frac{1}{L_{\max}} \leq \alpha_k  $. We precisely take this critical value $1 / L_{\max}$ for $\alpha_k$ at each step k in order to avoid a step too large to prevent oscillation phenomena. Our recursive relationship is now:
\begin{equation}
x_{k+1} = x_k - \frac{1}{L_{\max}} \left[\nabla f\left(x_k\right)\right]_{i_k} e_{i_k} \label{step_Lmax}
\end{equation}

\noindent Using the fact that $Var \left(  \left\| \nabla f \left( x_k \right)  \right\|  \right) \geq 0$, we can conclude that
\begin{flalign}
&\mathbb{E}  \left[  \left\| \nabla f \left( x_k \right)  \right\|^2  \right] \geq  \mathbb{E}  \left[  \left\| \nabla f \left( x_k \right)  \right\|  \right]^2   \label{inequality_expectation}
\end{flalign} 

\noindent By convexity of $f$, we have 
\begin{eqnarray}
f(x_k)-f\left(x^\star \right) &\leq& \langle \nabla f \left( x_k \right) \ , \ x_k - x^\star  \rangle \\
&\leq&  \left\| \nabla f \left( x_k \right)  \right\|  \left\| x_k - x^\star  \right\| \text{ (By Cauchy Schwartz)} \\
&\leq&  \left\| \nabla f \left( x_k \right)  \right\| R_0 \text{ (By \ref{R0} )} \label{main_proof_eq1}
\end{eqnarray}

\noindent We denote $U_k = \mathbb{E}\left[f \left(x_k \right)\right] - f \left(x^\star \right)$. Taking the expectation over all the random index $i_k$  in \ref{main_proof_eq1}, we obtain : 
\begin{equation}
U_k \leq  \mathbb{E}\left[  \left\| \nabla f \left( x_k \right)  \right\|\right] R_0
\end{equation}

\noindent $U_k \geq 0 $ for any k by definition of $x^\star$ which is the minima of f. So, taking the square value on both side of the inequality, we have 
\begin{eqnarray}
U_k^2 &\leq&  \mathbb{E}\left[  \left\| \nabla f \left( x_k \right)  \right\|\right]^2 R_0^2 \\
&\leq&  \mathbb{E}\left[  \left\| \nabla f \left( x_k \right)  \right\|^2 \right] R_0^2 \text{ by \ref{inequality_expectation}} \label{main_proof_inequality1}
\end{eqnarray}

\noindent Using the relation in \ref{step_Lmax}, we apply Taylor-Lagrange to our objective function f at step k+1 :
\begin{eqnarray}
f \left(x_{k+1}\right) &=& f \left(x_k\right) - \frac{1}{L_{\max}} \langle \left[ \nabla f\left(x_k \right) \right]_{i_k} e_{i_k} , \nabla f\left(x_k \right) \rangle \\
&+& \frac{1}{2} \left(\frac{1}{L_{\max}}\right)^2 \left(  \left[ \nabla f\left(x_k \right) \right]_{i_k} e_{i_k}  \right)^T \mathcal{H} f \left(d_k\right) \left(  \left[ \nabla f\left(x_k \right) \right]_{i_k} e_{i_k}  \right)
\end{eqnarray}

\noindent with $d_k \in \left] x_k , x_{k+1} \right[$. Therefore:
\begin{eqnarray}
 f \left( x_{k+1} \right)  &\leq&   f \left( x_k \right) - \frac{1}{L_{\max}}   \left\| \left[\nabla f \left( x_k \right)\right]_{i_k}  \right\|^2  \\
&&+ \frac{1}{2} \frac{1}{L_{\max}^2}  \left\| \left[\nabla f \left( x_k \right)\right]_{i_k}  \right\|^2   L_{\max} \\
&\leq&  f \left( x_k \right)  - \frac{1}{2 L_{\max}} \left\| \left[\nabla f \left( x_k \right)\right]_{i_k}  \right\|^2  
\end{eqnarray}

\noindent Taking the expectation over all the random indexes $i_k$, we have :
\begin{eqnarray}
\mathbb{E}\left[  \mathbb{E}_{i_k}\left[ f \left( x_{k+1} \right) \right]  \right] &\leq& \mathbb{E}\left[  \mathbb{E}_{i_k}\left[  f \left( x_k \right)  - \frac{1}{2 L_{\max}} \frac{1}{n} \sum_{i=1}^n \left[\nabla f \left( x_k \right)\right]_i^2  \right]  \right]  \\
\mathbb{E}\left[ f \left( x_{k+1} \right)\right]  &\leq& \mathbb{E}\left[  f \left( x_k \right) \right]  - \frac{1}{2n L_{\max}} \mathbb{E}\left[  \left\|  \left[\nabla f \left( x_k \right)\right] \right\|^2 \right]
\end{eqnarray}

\noindent Subtracting $f\left(x^\star \right)$ on both side, we obtain the main recursive relation between $U_k$ and $U_{k+1}$ :
\begin{equation}\label{main_result}
U_{k+1} \leq U_k - \frac{1}{2n L_{\max}}  \mathbb{E}\left[ \left\| \left[\nabla f \left( x_k \right)\right]  \right\|^2 \right]
\end{equation}

\noindent Using the inequality \ref{main_proof_inequality1}, we ensure that :
\begin{equation}
U_{k+1} \leq U_k - \frac{1}{2n L_{\max}}  \frac{1}{R_0^2} U_k^2
\end{equation}

\noindent We can conclude using lemma \ref{lemma1} to get $U_k \leq \frac{2n L_{\max}}{k}$ so that the first result of proposition \ref{theorem1} holds. 

\noindent The second result to prove \ref{theorem2} is also easy. Taking the minimum of both sides of \ref{strong_convexity} leads to $f^{\star} \geq f(x_k) - \frac{1}{2 \sigma} \left| \nabla f \left( x_k \right)\right|^2$ or equivalently $\left| \nabla f \left( x_k \right)\right|^2 \geq 2 \sigma U_{k} $. Then using \ref{main_result}, we get 
$$
U_{k+1} \leq U_k - \frac{\sigma}{n L_{\max}}  U_k = (1-\frac{\sigma}{n L_{\max}}) U_k.
$$
The result is trivially obtained by applying this formula recursively.
\end{proof}

\newpage

\bibliography{mybibfile}

\end{document}